\documentclass{article}
\usepackage{ijcai23}

\usepackage{times}
\usepackage{soul}
\usepackage{url}
\usepackage[hidelinks]{hyperref}
\usepackage[utf8]{inputenc}
\usepackage[small]{caption}
\usepackage{graphicx}
\usepackage{amsmath}
\usepackage{amsthm}
\usepackage{booktabs}
\usepackage{algorithm}
\usepackage[noend]{algorithmic}
\usepackage[switch]{lineno}
\usepackage{booktabs}       
\usepackage{amsfonts}       
\usepackage{nicefrac}       
\usepackage{microtype}      
\usepackage{xcolor}         
\usepackage{amsmath}
\usepackage{neuralnetwork2}
\usepackage{tikz} \usetikzlibrary{shapes.geometric}
\usepackage{multirow,array}
\usetikzlibrary {positioning}
\usepackage{dcolumn}
\usepackage{changepage}

\newtheorem{corollary}{Corollary}

\newtheorem{example}{Example}
\newtheorem{theorem}{Theorem}

\newcommand{\otr}[0]{\textsc{ReS}}
\newcommand{\odt}[0]{\textsc{ODT}}
\newcommand{\lcn}[0]{\textsc{LCN}}

\newcommand{\viper}[0]{\textsc{Viper}}
\newcommand{\lmviper}[0]{\textsc{LM-Viper}}
\newcommand{\ndps}[0]{\textsc{NDPS}}
\newcommand{\propel}[0]{\textsc{Propel}}

\newcommand{\prl}[0]{$\pi$-\textsc{PRL}}
\newcommand{\PIRL}[0]{\textsc{PIRL}}
\newcommand{\otrhat}[0]{$\widehat{\textsc{ReS}}$}

\title{Synthesizing Programmatic Policies \\ with Actor-Critic Algorithms and ReLU Networks}

\author{
Spyros Orfanos \and
Levi H. S. Lelis
\affiliations
Department of Computing Science, University of Alberta, Canada\\
Alberta Machine Intelligence Institute (Amii)\\
\emails
\{orfanos, levi.lelis\}@ualberta.ca
}

\begin{document}

\maketitle

\begin{abstract}
Programmatically Interpretable Reinforcement Learning (PIRL) encodes policies in human-readable computer programs. Novel algorithms were recently introduced with the goal of handling the lack of gradient signal to guide the search in the space of programmatic policies. Most of such PIRL algorithms first train a neural policy that is used as an oracle to guide the search in the programmatic space. In this paper, we show that such PIRL-specific algorithms are not needed, depending on the language used to encode the programmatic policies. This is because one can use actor-critic algorithms to directly obtain a programmatic policy. We use a connection between ReLU neural networks and oblique decision trees to translate the policy learned with actor-critic algorithms into programmatic policies. This translation from ReLU networks allows us to synthesize policies encoded in programs with if-then-else structures, linear transformations of the input values, and PID operations. Empirical results on several control problems show that this translation approach is capable of learning short and effective policies. Moreover, the translated policies are at least competitive and often far superior to the policies PIRL algorithms synthesize.   
\end{abstract}

\section{Introduction}

Recent work in programatically interpretable reinforcement learning (PIRL) encodes policies in computer programs~\cite{NDPS}. Such policies are more amenable to verification and tend to be easier to understand than neural policies~\cite{viper}. On the downside, one needs to search over large and often non-differentiable spaces of programs to synthesize programmatic policies. Algorithms such as \viper~\cite{viper}, \ndps~\cite{NDPS}, and \propel~\cite{propel} were introduced to search specifically in such spaces. These algorithms rely on the search signal that a neural oracle provides. For example, \ndps\ uses program synthesis~\cite{1969:PROW} to synthesize a ``sketch'' of a policy and then uses Bayesian optimization~\cite{BayesianOpt} to find suitable real-valued parameters for the sketch such that the resulting programmatic policy is as similar as possible to the oracle.  \propel\ trains a neural policy concurrently with the synthesis process to ensure that the neural policy is not ``too different'' from the programs being synthesized, so it can be more effective in guiding the programmatic search.  

In this paper, we show that actor-critic algorithms~\cite{actor-critic} offer an effective and simple alternative to existing PIRL algorithms if the language used to encode the policies only accounts for if-then-else structures, linear transformations of the input, and PID operations. 
 Instead of training a neural policy to serve as an oracle for the synthesis procedure, we simply train a neural policy with an actor-critic algorithm and translate it into a programmatic one. 

Our translation procedure uses a connection between ReLU neural networks and oblique decision trees (\odt)~\cite{guido14,liwen18,Lee2020Oblique}. For a given input, the neurons of a ReLU network are either active (produce a value greater than zero) or inactive (produce zero). An activation pattern defines which neurons are active and which neurons are inactive for a given input. Since a linear combination of linear functions is also linear, the network represents a linear function for a fixed activation pattern. Thus, an \odt\ can be used to represent the neural network, where each branch of the tree represents an activation pattern of the network and the height of the tree is equal to the number of nodes in the network. 

Since an \odt\ is a program with if-then-else structures and linear transformations of the input, the \odt\ we obtain from the translation procedure is, in fact, a program written in a language previously used to encode programmatic policies~\cite{viper,NDPS,propel}. We use large networks for the critic and small ones for the actor as it is the latter that is translated into a program, and smaller networks result in shorter and often more interpretable programs. We call ReLU Synthesizer (\otr) this translation method and we evaluate it on several control problems. Empirical results show that \otr\ is at least competitive and often far superior to other PIRL algorithms in all evaluated domains.  

The contribution of this paper is to show that the translation of ReLU networks to \odt\ can be used to synthesize short and effective policies to challenging problems. This translation offers a drastically simpler method to PIRL if the language used to encode the policies only accounts for if-then-else structures, linear transformations of the input, and PID operations. 

\section{Related Work}

Other PIRL works explored languages more expressive than those considering only if-then-else structures. For example, \citeauthor{koul2018learning}~\shortcite{koul2018learning} and \citeauthor{Inala2020Synthesizing}~\shortcite{Inala2020Synthesizing} considered finite-state machines for encoding policies. Such finite-state machines can be translated into programs with loops, which can be helpful, as \citeauthor{Inala2020Synthesizing} showed, to allow the policy to generalize to unseen scenarios. Although they were evaluated with languages that support only if-then-else structures, the algorithms used in \ndps~\cite{NDPS} and \propel~\cite{propel} can potentially be used to search for programmatic strategies in spaces supporting loops. Due to the mapping from ReLU networks, \otr\ is unable to synthesize policies with loops. 

Programmatic policies can also be obtained with gradient descent and soft decision trees, such as Differentiable Decision Trees (DDT)~\cite{Silva20a}, Cascading Decision Trees (CDT)~\cite{CDT}, and Programatically Interpretable Reinforcement Learning (\prl) \cite{qiu2022programmatic}. We use \prl\ in our experiments because it presents much stronger results than DDT and CDT. \prl\ uses policy gradient algorithms to learn policies encoded as decision trees. The drawback of \prl\ is that the time complexity to evaluate a state during training is exponential on the height of the tree. Thus, it can only synthesize small programs (\citeauthor{qiu2022programmatic} induced trees of depth 6 in their experiments).

Locally constant networks (\lcn) also exploit the fact that each activation pattern of a neural network maps to a linear function of the input to induce trees from ReLU neural networks in the context of supervised learning~\cite{Lee2020Oblique}. In contrast to the mapping we present in this paper that directly translates a ReLU neural network to an \odt, \lcn\ learns a function of the derivatives of the ReLU neural networks. As a result, \lcn\ requires a more complex procedure for training the model. Namely, the gradient signal for training the model with gradient descent is insufficient, and the model is trained with an annealing procedure where it transitions from a smooth approximation of ReLU to ReLUs. \lcn's training procedure uses a dynamic programming procedure to compute all gradients efficiently because regular backpropagation is not efficient for training such models. We do not use \lcn\ in our experiments because it is not clear how to deal with its annealing procedure in the context of online learning. 

The relationship between ReLU networks and \odt\ was observed in works prior to \lcn~\cite{guido14,liwen18}. However, none of these previous work explored the use of this relationship to synthesize programmatic policies.

\section{Background and Notation}

\begin{figure*}[t!]
    \centering
    \begin{tabular}{ccc} 
    \toprule
\textbf{Neural Network} & \textbf{Weights and Biases} & \textbf{Inference Path} \\
\midrule
\begin{minipage}[h]{4.0cm}
    \begin{neuralnetwork}[height=3]
        \newcommand{\hinput}[2]{$x_#2$}
        \newcommand{\hfirst}[2]{$A^{2}_#2$}
        \newcommand{\hsecond}[2]{$A^{3}_#2$}
        \inputlayer[count=2, bias=false, text=\hinput]
        \hiddenlayer[count=3, bias=false, text=\hfirst] \linklayers
        \outputlayer[count=1, text=\hsecond] \linklayers
    \end{neuralnetwork}
\end{minipage}
&\begin{minipage}[h]{6.0cm} $\begin{aligned} 
W^{1} &= \begin{bmatrix} -2.7 & -0.8 \\
0.2& 2.0 \\ 
1.0 & -0.1 \end{bmatrix}, B^{1}= \begin{bmatrix} -0.4 \\
0.6 \\ 
1.2 \end{bmatrix} \\
W^{2} &= \begin{bmatrix} -2.0 & -2.4 & 1.2 \\ \end{bmatrix}, B^{2}= \begin{bmatrix} 1.4 \end{bmatrix} \\
P^{1} &= \begin{bmatrix} 0 & 0 \\
0 & 0 \\ 
1.0 & -0.1 \end{bmatrix}, V^{1} = \begin{bmatrix} 0 \\
0 \\ 
1.2 \end{bmatrix} \\
P^{2} &= \begin{bmatrix} 1.2 & -0.12 \end{bmatrix}, V^{2} = \begin{bmatrix} 2.84 \end{bmatrix}
\end{aligned}$\end{minipage}
&
\begin{minipage}[h]{4.0cm}
\begin {tikzpicture}[
        shorten >=1pt, auto, thick,
        node distance=0.7cm,
state/.style={circle,draw},
inference/.style={fill=black!20},
square/.style={rectangle}]
  \tikzstyle{level 1}=[sibling distance=2mm] 
  \tikzstyle{level 2}=[sibling distance=2mm] 
  \tikzstyle{level 3}=[sibling distance=2mm] 
\node[state, square, inference] (A) [] {\footnotesize $Z'^{2}_1$};
\node[state, square, inference] (B) [below left=of A] {\footnotesize $Z'^{2}_2$};
\node[state, square, inference] (C) [below left=of B] {\footnotesize $Z'^{2}_3$};
\node[state, square, inference] (D) [below right=of C] {\footnotesize $Z'^{3}_1$};
\path (A) edge node [left] {\footnotesize $\leq 0$} (B);
\path (B) edge node [left] {\footnotesize $\leq 0$} (C);
\path (C) edge node [right] {\footnotesize $> 0$} (D);
\end{tikzpicture}
\end{minipage}
 \\
 \midrule
 \multirow{2}{*}{$\mathbf{Z'}$\textbf{-functions}} & \multicolumn{2}{l}{$Z'^{2}_1 = -2.7 x_1 -0.8 x_2 - 0.4, Z'^{2}_2 = 0.2 x_1 + 2.0 x_2 + 0.6$} \\
& \multicolumn{2}{l}{$Z'^{2}_3 = 1.0 x_1 - 0.1 x_2 + 1.2, Z'^{3}_1 = 1.2 x_1 - 0.12 x_2 + 2.84$} \\

 \bottomrule
    \end{tabular}
    \caption{Neural network (left); set of weights and biases ($W$ and $B$) of the network and of the inference path for input $X = [0.5 \,\, -0.5]^T$ of the tree \otr\ induces ($P$ and $V$) (middle); inference path for $X$ (right). Each $Z'^{i}_k$-value (bottom) is equal to its corresponding $Z^{i}_k$-value for the input $X$.
    }
    \label{fig:example}
\end{figure*}
We denote matrices with upper-case letters and scalar values with lower-case letters. We consider fully connected neural networks with $m$ layers ($1, \cdots, m$), where the first layer is given by the input values $X$ and the $m$-th layer the output of the network. For example, $m=3$ for the network shown in Figure~\ref{fig:example}. Each layer $j$ has $n_j$ neurons ($1, \cdots, n_j$) where $n_1 = |X|$. 
The parameters between layers $i$ and $i+1$ of the network are denoted by $W^{i} \in \mathbb{R}^{n_{i+1} \times n_i}$ and $B^{i}  \in \mathbb{R}^{n_i \times 1}$. 
The $k$-th row vector of $W^{i}$ and $B^{i}$, denoted $W^{i}_k$ and $B^{i}_k$, represent the weights and the bias term of the $k$-th neuron of the $(i+1)$-th layer. Figure~\ref{fig:example} shows an example where $n_1 = 2$ and $n_2 = 3$. Thus, $W^1 \in \mathbb{R}^{3 \times 2}$ and the first row vector of $W^1$ ($[-2.7 \,\, -0.8]$) and the first entry of $B^1$ ($-0.4$) provide the weights and bias of the first neuron in the hidden layer of the model. 
%
Let $A^{i} \in \mathbb{R}^{n_i \times 1}$ be the values of the neurons of the $i$-th layer, where $A^{1} = X$ and $A^{m}$ is the output of the model. A forward pass in the model computes the values of $A^{i} = g(Z^{i})$, where $g(\cdot)$ is an activation function, and $Z^{i} = W^{i-1} \cdot A^{i-1} + B^{i-1}$. We compute the values of $A^{i}$ in the order of $i = 2, \cdots, m$. 

We consider ReLU activation functions: $\text{ReLU}(x) = \max(0, x)$~\cite{relu} for all neurons in hidden layers; we consider Logistic and linear activation functions for neurons in the output layer for classification and regression problems, respectively. \otr\ also works with Leaky ReLUs~\cite{leaky-relus} in hidden layers and Softmax in the output layer; we discuss the use of such functions in the next section. 

\begin{figure}
    \begin{align*}
        E ::= \; & C \; | \; \textbf{if} \; B \; \textbf{then} \; E \; \textbf{else}  \; E \\ 
        B ::= \; & P \cdot X + v \leq 0 \\
        C ::= \; & P \cdot X + v
    \end{align*}
    \vspace{-7mm}
    \caption{DSL for regression oblique decision trees}
    \label{fig:dsl}
\end{figure}

An oblique decision tree $T$ is a binary tree whose nodes $s$ define a function $P \cdot X + v \leq 0$ of the input $X$, where $P$ and $v$ are parameters of $s$. Each leaf of $T$ contains a prediction for $X$. 
For classification tasks, the leaves return a label. 
For regression tasks, we consider linear model trees \cite{linear-model-trees}, where each leaf returns the value of $X \cdot P + v$ as the prediction for $X$. 
An oblique tree for regression is equivalent to the programs written in a commonly used DSL from the programmatic policy literature (see Figure~\ref{fig:dsl} for the DSL). 
A tree $T$ produces a prediction for $X$ by defining a path from the root to a leaf of $T$, which we call an \emph{inference path}. An inference path is determined as follows. If $P \cdot X + v \leq 0$ is true for the root, then one follows to the root's left child; the right child is followed otherwise. This rule is applied until we reach a leaf node. 


\section{ReLU Networks to Oblique Decision Trees}
\label{sec:oblique}


In this section, we present Oblique Trees from ReLU Neural Networks (\otr), an algorithm for inducing an oblique tree that is equivalent to a fully connected neural network that uses ReLU activation functions. In our presentation we assume a single neuron in the output layer that uses either a Logistic function (classification tasks) or a linear function (regression tasks). 
Later we show how \otr\ generalizes to multi-class tasks where the network uses a Softmax function in its output layer. 

\otr\ leverages the metric of activation patterns of a neural network~\cite{raghu2017}, which considers the ``active'' and ``inactive'' neurons for a fixed input. For ReLU functions, the $k$-th neuron of the $i$-th layer is active if $Z_k^{i} > 0$ and is inactive otherwise. 
A network with $n$ neurons results in a tree with depth $n$, where each node on a path from root to leaf represents a different neuron and each path represents an activation pattern.
Once an activation pattern is defined, the $Z^i_k$-values of the network are simply a linear transformation of the input $X$~\cite{Lee2020Oblique}. 
\otr\ chooses the parameters $P$ and $v$ of each node such that $Z^i_k = P \cdot X + v$ for the node representing the $k$-th neuron of the $i$-th layer.  

Algorithms~\ref{alg:or} and \ref{alg:recursive_otr} show the pseudocode of \otr, which receives the weights and biases of a network $R$ and the problem type $t$ (either binary classification or regression); \otr\ returns an oblique decision tree equivalent to $R$. \otr\ recursively processes all neurons of layer $i$ before processing neurons of layer $i+1$. 
It starts with the root of the tree, which 
represents the first neuron of layer $j=2$ ($l = 1$ and $k = 1$ in the pseudocode) and finishes with the output neuron. 

\otr\ rewrites the functions $Z$ of the network in terms of input values $X$, the resulting functions are denoted $Z'$. Similarly to how the $Z$-values are computed in terms of $W$ and $B$, the $Z'$-values are computed in terms of the matrices $P$ and $V$. We define $Z'^i_k = P^{i-1}_k \cdot X + V^{i-1}_k$, where $P^{i-1}_k$ and $V^{i-1}_k$ are the $k$-th row vector of $P^{i-1}$ and $V^{i-1}$. 
Once the values of $P_k^i$ and $V_k^i$ are first computed (line~\ref{line:init_p_v} of Algorithm~\ref{alg:or} or lines~\ref{line:transform_w} and \ref{line:transform_b} of Algorithm~\ref{alg:recursive_otr}), $Z'^{i+1}_k = Z^{i+1}_k$. 
However, the value $Z'^{i+1}_k$ of a node must be equal to $A^{i+1}_k$, so that the weights $P^{i+1}$ and $V^{i+1}$ of the next layer can be computed (in our example $P_k^1$ and $V_k^1$ are such that $A^{i+1}_k = Z'^{i+1}_k = P^i_k \cdot X + V^i_k$). In line~\ref{line:set_to_zero} of Algorithm~\ref{alg:recursive_otr} the values of $P^i_k$ and $V^i_k$ are set to zero, so that the $Z'^{i+1}_k$ computed from the $V$ and $P$ matrices passed as parameters to the recursive calls in lines~\ref{line:recurse_left} and \ref{line:recurse_right} are equal to $A^{i+1}_k$. The first recursive call treats the case where $A^{i+1}_k = Z^{i+1}_k$ and recursively creates the right child of node $r$. The second call treats the case where $A^{i+1}_k = 0$ and recursively creates the left child of $r$. 

The matrices $P^1$ and $V^1$ are equal to $W^1$ and $B^1$ (line~\ref{line:init_p_v}) because the functions $Z^2$ are defined in terms of $X$. 
Matrices $P^l$ for $l > 1$ are computed in line~\ref{line:transform_w} with the operation $W^l_k \cdot P^{l-1}$. This operation performs a weighted sum of the values $p_j$ 
of the neurons $q$ from the previous layer; the sum is weighted by the value in $W$ representing the connection between the neuron $k$ being processed with neuron $q$ from the previous layer. In our example, this operation was used to compute $P^2 = W^2_1 \cdot P^1 = [1.2 \,\, -0.12]$. Similarly, $V^l$ is computed with the operation $W^l_k \cdot V^{l-1} + B^l_k$, which is a weighted sum of the bias terms of the neurons in the previous layer, added to the bias term of the current neuron. In our example, we computed $V^2 = W^2_1 \cdot V^1 + B^2_1 = 2.84$. Once the $k$-th rows of $P^l$ and $V^l$ are computed, we create the node representing the $k$-th neuron of layer $l+1$ with the parameters $P_k^l$ and $V_k^l$ (line~\ref{line:create_node}).

For classification tasks, 
the left child of the node representing the output neuron returns the label $0$, while its right child returns the label $1$ (line~\ref{line:classification}). 
For regression tasks, the node representing the output neuron is a leaf and returns the value of $P^{m-1}_1 \cdot X + V^{m-1}_1$ as the prediction for $X$. As one can note, for classification tasks, the tree has one extra level if one considers the nodes with labels. 

\begin{example}
Consider the network $R$ in Figure~\ref{fig:example}. 
\otr\ induces an oblique decision tree $T$ that is equivalent to $R$ (the value of the leaf node in the inference path for fixed values of $x_1$ and $x_2$ is equal to the output of $R$ for the same inputs). Figure~\ref{fig:example} shows the inference path (right-hand side) for $x_1 = 0.5$ and $x_2 = -0.5$ of the tree \otr\ induces from $R$. 
Each node on an inference path of $T$ defines a function $Z'^i_{k}$ of the input values that matches the value of $Z^i_{k}$ of $R$ for a fixed input (we omit the parameters of the functions to simplify notation). 
The $Z'$ functions of the nodes representing neurons in the layer $i=2$ of $R$ are equal to the $R$'s $Z$-functions because $Z$ is already defined in terms of the input. Here, $Z'^{2}_1 = -2.7 x_1 -0.8 x_2 - 0.4 = Z^{2}_1$, $Z'^{2}_2 = 0.2 x_1 + 2.0 x_2 - 0.6 = Z^{2}_2$, and $Z'^{2}_3 = 1.0 x_1 - 0.1 x_2 + 1.2 = Z^{2}_3$. The value of $Z^{3}_1$ is computed in $R$ according to the output values of the neurons in layer $i=2$. \otr\ defines $Z'^{3}_1 = Z^{3}_1$ in terms of $x_1$ and $x_2$ as follows. 

We define matrices $P^{1}$ and $V^{1}$ where the $k$-th row of $P^{1}$ defines the weights of the $k$-th neuron of layer $2$ in terms of $x_1$ and $x_2$. Similarly, the $k$-th entry of $V^1$ defines the bias term of the $k$-th neuron when the value the neuron produces is written in terms of $x_1$ and $x_2$. In this example, for fixed $x_1 = 0.5$ and $x_2 = -0.5$, $P^{1}$ and $V^{1}$ are $W^{1}$ and $B^{1}$ with the first two rows filled with zeros. This is because the first two neurons of layer $i=2$ are inactive for $x_1$ and $x_2$. The weights of $Z'^{3}_1$ are given by $P^{2}_1$, where $P^{2}_1 = W^{2}_1 \cdot P^{1}$ and the bias term by $V^{2}_1$, where $V^{2}_1 = W^{2}_1 \cdot V^{1} + B^{2}_1$ (Figure~\ref{fig:example} shows the matrices $P^2$ and $V^2$). As can be verified, $Z'^{3}_1 = Z^{3}_1$ for $x_1 = 0.5$ and $x_2 = -0.5$. For the decision tree: $Z'^{3}_1 = 1.2 \cdot 0.5 - 0.12 \cdot (-0.5) + 2.84 = 3.5$. For the neural network: 
\begin{align*}
    Z^2 &= \begin{bmatrix} -2.7 & -0.8 \\
0.2& 2.0 \\ 
1.0 & -0.1 \end{bmatrix} \cdot \begin{bmatrix} 0.5 \\
-0.5 \end{bmatrix} + \begin{bmatrix} -0.4 \\
0.6 \\ 
1.2 \end{bmatrix} = \begin{bmatrix} -1.35 \\
-0.30 \\ 
1.75 \end{bmatrix} \\
A^2 &= \begin{bmatrix} 0 \\
0 \\ 
1.75 \end{bmatrix} \\
Z^3 &= \begin{bmatrix} -2.0 & -2.4 & 1.2 \\ \end{bmatrix} \cdot \begin{bmatrix} 0 \\
0 \\ 
1.75 \end{bmatrix} + \begin{bmatrix} 1.4 \end{bmatrix} = 3.5
\end{align*}
If the output neuron uses a linear function, then the output of network $A^{3}_1 = Z^{3}_1 = Z'^{3}_1$; if it uses a Logistic function $g$, then $T$ predicts class $1$ because $Z'^{3}_1 = 3.5$ and $g(x) > 0.5$ if and only if $x > 0$ for the Logistic function; it would predict class $0$ if $Z'^{3}_1 \leq 0$. 
\end{example}

\begin{algorithm}[!t]
\caption{{\sc Oblique Trees from ReLU Networks (OTR)}}
\label{alg:or}
\begin{algorithmic}[1]
\REQUIRE Neural Network's Weights $W$ and biases $B$, problem type $t$ (classification or regression)
\ENSURE Oblique tree $T$.
\STATE Initialize $P$ as a set of matrices $P^{1}, P^{2}, \cdots, P^{m-1}$, where $P^{i} \in \mathbb{R}^{n_{i+1} \times n_{1}}$.  
\STATE Initialize $V$ as a set of matrices $V^{1}, V^{2}, \cdots, V^{m-1}$, where $V^{i} \in \mathbb{R}^{n_{i+1} \times 1}$.
\STATE $P^{1} \gets W^{1}$, $V^{1} \gets B^{1}$ \label{line:init_p_v}
\STATE $r \gets$ {\sc Empty-Node} \label{line:init_root}
\STATE {\sc Induce-Oblique-DT}($r$, $t$, $1$, $1$, $W$, $B$, $P$, $V$) \label{line:first_call}
\RETURN $r$
\end{algorithmic}
\end{algorithm}

\begin{algorithm}[!htb]
\caption{{\sc Induce Oblique DT}}
\label{alg:recursive_otr}
\begin{algorithmic}[1]
\REQUIRE Node $r$, problem type $t$, layer $l$ and neuron $k$, matrices $W$, $B$, $P$, $V$.
\IF{$l > 1$}
\STATE $P^{l}_k \gets W^{l}_k \cdot P^{l-1}$ \label{line:transform_w}
\STATE $V^{l}_k \gets W^{l}_k \cdot V^{l-1} + B^{l}_k$ \label{line:transform_b}
\ENDIF
\IF{$l = m - 1$} \label{line:leaf_nodes}
\IF{$t$ is classification}
\STATE $r \gets (P^{l}_k, V^{l}_k, \text{left}=0, \text{right}=1)$ \label{line:classification} 
\ENDIF
\IF{$t$ is regression}
\STATE $r \gets (P^{l}_k, V^{l}_k)$ \label{line:regression} 
\ENDIF
\RETURN \label{line:end_leaf_nodes}
\ENDIF
\STATE $e \gets $ {\sc Empty-Node}, $d \gets $ {\sc Empty-Node}
\STATE $r \gets (P^{l}_k, V^{l}_k, \text{left}=e, \text{right}=d)$ \label{line:create_node} 
\IF{$k = n_l$} \label{line:end_layer}
\STATE $k \gets 1$
\STATE $l \gets l + 1$
\ELSE
\STATE $k \gets k + 1$ \label{line:increment_neuron}
\ENDIF
\STATE {\sc Induce-Oblique-DT}($d$, $t$, $l$, $k$, $W$, $B$, $P$, $V$) \label{line:recurse_right}
\STATE $P^{l}_{k} \gets \begin{bmatrix} 0 & \cdots & 0 \end{bmatrix}$, $V^{l}_{k} \gets \begin{bmatrix} 0 \end{bmatrix}$ \label{line:set_to_zero}
\STATE {\sc Induce-Oblique-DT}($e$, $t$, $l$, $k$, $W$, $B$, $P$, $V$) \label{line:recurse_left}
\end{algorithmic}
\end{algorithm}

\begin{theorem}
Let $W$ and $B$ be the weights and biases of a fully connected neural network $R$ whose hidden-layer units use ReLU activation functions and the single unit of its output layer uses a Logistic or a linear activation function. The oblique decision tree $T$ \otr\ induces with $W$ and $B$ is equivalent to $R$, that is, $T$ and $R$ produce the same output for any input $X$.  
\end{theorem}
\begin{proof}
For a fixed input $X$, we prove $Z'^{i}_k = A^{i}_k$, where $Z'^{i}_k = P_k^{i-1} \cdot X + V_k^{i-1}$ and $A^{i}_k = \text{ReLU}(Z^{i}_k)$ for all values of $Z'^{i}_k$ encountered along the inference path of $X$ in $T$. 
If $Z'^{i}_k = A^{i}_k$ for any $i$ and $k$, the $Z'$-value of the leaf node on the inference path matches the output of $R$ for a fixed $X$. Thus, both $T$ and $R$ produce the same output for any fixed input $X$. 
 
Our proof is by induction on the layer $i$. The base case considers $i=2$, the first layer of the model. 
\begin{align*}
Z_k^{2} &= W_k^{1} \cdot X + B_k^{1} \,(\text{definition of } Z^{2})  \\ 
&= P_k^{1} \cdot X + V_k^{1} \,(\text{line~\ref{line:init_p_v} of Algorithm~\ref{alg:or}})
\end{align*}
During inference, if $P_k^{1} \cdot X + V_k^{1} \leq 0$, we follow the left child of the node with parameters $P_k^{1}$ and $V_k^{1}$. In this case, $P_k^{1}$ is set to a vector of zeros and $V_k^{1}$ is set to zero (line~\ref{line:set_to_zero} of Algorithm~\ref{alg:recursive_otr}), thus $P_k^{1} \cdot X + V_k^{1} = 0$. If $P_k^{1} \cdot X + V_k^{1} > 0$, then we follow the right child and $P_k^{1} \cdot X + V_k^{1} = W_k^{1} \cdot X + B_k^{1}$. Therefore, $Z'^{2}_k = A^{2}_k$, for any $k$. 
The inductive hypothesis assumes that $A^{i-1}_k = P_k^{i-2} \cdot X + V_k^{i-2}$.

For the inductive step we have the following. 
\begin{align}
Z'^{i}_k &= P_k^{i-1} \cdot X + V_k^{i-1} \label{eq:definition_a} \\
&= (W^{i-1}_k \cdot P_k^{i-2}) \cdot X + W^{i-1}_k \cdot V_k^{i-2} + B^{i-1}_k \label{eq:pseudocode} \\
&= W^{i-1}_k (P_k^{i-2} \cdot X + V_k^{i-2}) + B^{i-1}_k \label{eq:rearranging} \\
&= W^{i-1}_k A_k^{i-1} + B^{i-1}_k \label{eq:inductive_step} \\
&= Z_k^{i} \label{eq:definition_z}
\end{align}
Step 1 
uses the definition of $Z'^{i}_k$, while step 2 
is due to the computation in lines~\ref{line:transform_w} and \ref{line:transform_b} of Algorithm~\ref{alg:recursive_otr}. 
Step 3 
uses the inductive hypothesis and 
step 4 
the definition of $Z^i_k$.

We consider the two possible cases for $Z^i_k$:
\begin{enumerate}
\item $Z^i_k \leq 0$: \otr\ sets $P_k^{i-1}$ and $V_k^{i-1}$ to zeros (line~\ref{line:set_to_zero} of Algorithm~\ref{alg:recursive_otr}) so $Z'^i_k = Z^i_k = A^i_k = 0$.
\item $Z^i_k > 0$: we have from the derivation above that $Z'^i_k = Z^i_k = A^i_k$. 
\end{enumerate}


Thus, $Z'^i_k = A^i_k$. 

The parameters $P_1^{m-1}$ and $V_1^{m-1}$ of leaf nodes are never set to zero because line~\ref{line:set_to_zero} of Algorithm~\ref{alg:recursive_otr} is not reached for them. Therefore, $Z'^{m}_1 = Z_1^{m} = P_1^{m-1} \cdot X + V_1^{m-1}$. 
In regression tasks, $T$ returns the value $Z'^{m}_1 = Z_1^{m}$ of the leaf node as its prediction (line~\ref{line:regression}). In classification tasks, the leaf node with label $0$ is reached if $P_1^{m-1} \cdot X + V_1^{m-1} \leq 0$ for the node representing the output neuron because $g(z) \leq 0.5$ if and only if $z \leq 0$ for the Logistic function $g$; the leaf node with label $1$ is reached otherwise. Therefore, $T$ and $R$ produce the same output for a fixed input $X$. 
\end{proof}

\otr\ is also able to induce oblique decision trees that are equivalent to densely connected networks~\cite{densely-connected}. In densely connected networks, every neuron in layer $i$ receives as input the output of all neurons in layers $j < i$. The values of $P_k^l$ and $V_k^l$ must be appended to the matrices $P^i$ and $V^i$ for $i > l$ because the output of the $k$-th neuron of layer $l$ is used as input in all the following layers. 

\begin{corollary}
Let $W$ and $B$ be the weights and biases of a densely connected neural network $R$ whose hidden layer units use ReLU activation functions and the single unit of its output layer uses a Logistic or a linear activation function. The oblique decision tree $T$ \otr\ induces with $W$ and $B$ is equivalent to $R$, that is, $T$ and $R$ produce the same output for any input $X$.  
\end{corollary}

\paragraph{Leaky ReLU.} 
\otr\ can be modified to handle Leaky ReLU functions: $\text{LReLU}(x) = \max(x, a \cdot x)$, where $0 < a < 1$. Similarly to ReLUs, the right child of a node handles the $A^i_k = Z^i_k$ case, and the left child handles the $A^i_k = a \cdot Z^i_k$ case. Instead of setting the values of $P_k^l$ and $V_k^l$ to zero in line~\ref{line:set_to_zero} of Algorithm~\ref{alg:recursive_otr}, \otr\ assigns the values of $a (W^{l}_k \cdot P^{l-1})$ to $P^l_k$ and $a(W^{l}_k \cdot V^{l-1} + B^{l}_k)$ to $V^l_k$. 

\paragraph{Multi-Class Tasks.} 
For multi-class tasks \otr\ can handle networks with multiple neurons in the output layer and Softmax functions. This is achieved by implementing, as part of the tree, a maximum function for the $Z'$-values of the output neurons. For example, a node $s$ checks if $Z'^{m-1}_1 - Z'^{m-1}_2 \leq 0$ (is output $2$ larger than output $1$?) and $s$'s left child checks $Z'^{m-1}_2 - Z'^{m-1}_3 \leq 0$ (is output $3$ larger than output $2$?) while its right child checks $Z'^{m-1}_1 - Z'^{m-1}_3 \leq 0$, etc. Another way of handling multi-class tasks is to train one model for each label. 

\paragraph{Sparse Oblique Trees.} Axis-aligned decision trees 
tend to be easier to interpret than oblique trees because each axis-aligned node considers a single feature (e.g., $x_i \leq b$). By contrast, each node in an oblique tree considers all features. In sparse oblique trees, some of the weights $p_i$ related to $x_i$ are set to zero, thus increasing interpretability~\cite{TAO}. 
\otr\ allows for the induction of sparse trees if one uses L1-regularization while training the underlying neural network~\cite{tibshirani96regression}. L1 regularization is effective in inducing sparse oblique trees with \otr\ if the network has a single hidden layer. Since L1 regularization is able to drive some of the weights of the model toward zero, 
nodes representing neurons in the model's first hidden layer will have some of its $p_i$-values also set to zero (because $P^1 = W^1$). The nodes that represent neurons in layers $i > 2$ are less likely to be sparse as they depend on a combination of $w$-values being set to zero or adding up to zero in the operation $W^i_k \cdot P^{i-1}$. 

\paragraph{Pruning and Approximation.} 
Recent work showed that ReLU networks have surprisingly few activation patterns compared to the maximum possible \cite{activ_region}. \citeauthor{Lee2020Oblique}~\shortcite{Lee2020Oblique} also observed this phenomenon in their \lcn\ experiments. Activation patterns that are not encountered (or rarely encountered) on a sufficiently large data set can possibly be removed with little to no effect on the model's performance. For \otr\ trees, this translates into pruning entire inference paths, which can significantly reduce the complexity of the learned model and possibly increase their interpretability. We propose \otrhat, an approximation of \otr\ in which we prune branches that are not reached while executing the neural model in the environment. We also propose \otrhat($k$), where only the $k$ inference paths most frequently encountered are kept and all others are removed. As we observe in our experiments, for some problems, training a larger \otr\ model and then approximating it using \otrhat\ is easier than training an \otr\ model of reduced size.

\paragraph{PID Controller Policies.}
Proportional-integral-derivative (PID) controllers have long been used to stabilize control systems due to their robustness and stability guarantees. More recently, discretized PID controllers have been used in \PIRL\ \cite{NDPS,propel,qiu2022programmatic} where the proportional ($P$), integral ($I$), and derivative ($D$) are approximated as follows: \linebreak $P = (\epsilon - s), I = \textbf{fold}(+, \epsilon - h), D = \textbf{peek}(h, -1) - s$, where $\epsilon$ is the known fixed target for which the system is stable, $s$ is the current state of the environment, $h$ is a history of the previous $k$ states (we use $k=5$), \textbf{fold} is a higher-order function which sums the input sequence along its dimension, and \textbf{peek}($h$, -1) returns the last state in $h$. 

The only change to the DSL in Figure \ref{fig:dsl} is to the controller symbol, which is replaced by $C ::=  \theta_P \cdot P + \theta_I \cdot I + \theta_D \cdot D$. Neural networks can learn these parameters, denoted $\theta_P, \theta_I$ and $\theta_D$, which means that we can use \otr\ to synthesize a PID controller policy. However, the learned parameters are functions of the input (e.g., $\theta_P = \theta_P(s) 
$), which is more expressive than the PID controller DSL from other works. 

\section{Empirical Evaluation}

In this section, we describe our empirical methodology for evaluating \otr\ as a means of synthesizing programmatic policies to solve control problems. 

\paragraph{Problem Domains.} 
We evaluate \otr\ with depth-six oblique decision trees on eight continuous action control problems from OpenAI Gym \cite{1606.01540} and MuJoCo \cite{todorov2012mujoco}: Reacher, Walker2D, Hopper, HalfCheetah (HC), Ant, Swimmer, BipedalWalker (BW), and Pendulum. With the exception of Pendulum, the action spaces are multi-dimensional, so each action dimension has a linear model in the leaf nodes. We use depth-six trees so our results are comparable to those of \prl, which also uses depth-six trees. The reinforcement learning models were implemented in the Stable Baselines3 repository~\cite{stable-baselines3}. The experiments for \otr\ were completed in approximately 504 hours on a single CPU. 

\paragraph{Learning Algorithms.}
Since we are interested in finding a programmatic policy representation and not a programmatic value function, we use actor-critic methods to train small neural network policies and arbitrarily sized value networks, and apply \otr\ on the policy network only. We run PPO \cite{https://doi.org/10.48550/arxiv.1707.06347} for 3 million steps in BW and Swimmer, SAC \cite{haarnoja2017soft} for 3 million steps in Pendulum and Reacher for 4 million steps in Walker2D, Hopper, HC, and Ant. We `squish' the actions the model produces by using a hyperbolic tangent function in the leaf nodes when using SAC. 

\paragraph{Hyperparameters.} We use hyperparameter values that are similar to the default values of known open source implementations. For BW and Swimmer, we use the default hyperparameters from Stable-Baselines3 for PPO: learning rate of $0.0003$, minibatch size of $64$, GAE parameter of $0.95$, $10$ epochs, clip factor of $0.2$, and discount factor of $0.99$. For Swimmer, we use a discount factor of $0.9999$. For Pendulum, Reacher, HC, Hopper, Walker2d, and Ant we use SAS with hyperparameter values similar to the default values for TD3 in Stable-Baselines3: discount factor of $0.99$, learning rate of $0.001$, minibatch size of $100$, buffer size of $1000000$, $10000$ learning starts and added Gaussian noise $\mathcal{N}(0, 0.1)$ to actions. We use ReLU activation functions and critic networks with two hidden layers of size 256 each on all domains except Walker2d, for which we use LeakyReLU activation functions.

\paragraph{Baselines. }
We use several algorithms from the PIRL literature as baselines in our experiments. Namely, we use \ndps~\cite{NDPS}, \prl~\cite{qiu2022programmatic} and \viper~\cite{viper}, and a modified version of \viper\ that uses linear model trees~\cite{linear-model-trees} (\lmviper) as the baselines. \ndps\ and \prl\ use the language from Figure \ref{fig:dsl}, \viper\ and \lmviper\ produce axis-aligned trees. We use \ndps\ and not \propel~\cite{propel} because previous work noted that the former performs better than the latter on the OpenAIGym and Mujoco domains~\cite{qiu2022programmatic}. We also train ReLU networks with 32 neurons and apply \otrhat\ to demonstrate the ability of \otr\ to synthesize longer programs. For each domain, we perform three independent runs (5 for \otr) of each system and evaluate the best policy found during training for 100 consecutive episodes. 

\begin{table*}[h!]
\small
\setlength{\tabcolsep}{4pt}
\centering
\caption{Reward and standard deviation of depth-6 policies over 100 episodes, averaged over three (five for \otr) independent runs of each algorithm. The last column shows the performance for depth-32 \otrhat\ policies. In parentheses is the estimated pruned depth, log$_2a$ ($a$ is the number of activation patterns realized over 100 episodes). The best average for each domain is highlighted in bold.}
\begin {center}
\begin{tabular}{
l@{\hspace{-1mm}} 
D{,}{\hspace{0.5mm}\pm\hspace{0.5mm}}{-1}@{\hspace{-0.75mm}} 
D{,}{\hspace{0.5mm}\pm\hspace{0.5mm}}{-1}@{\hspace{-0.75mm}} 
D{,}{\hspace{0.5mm}\pm\hspace{0.5mm}}{-1}@{\hspace{0mm}} 
D{,}{\hspace{0.5mm}\pm\hspace{0.5mm}}{-1}@{\hspace{0.2mm}} 
D{,}{\hspace{0.5mm}\pm\hspace{0.5mm}}{-1}@{\hspace{0.2mm}}  
D{,}{\hspace{0.5mm}\pm\hspace{0.5mm}}{-1}@{\hspace{-0.2mm}} 
c}
\toprule
\text{\small Environment}
& \multicolumn{1}{c}{\text{\viper}}
& \multicolumn{1}{c}{\text{\lmviper}}
& \multicolumn{1}{c}{\text{\ndps}}
& \multicolumn{1}{c}{\text{\prl}}
& \multicolumn{1}{c}{\text{\otr}}
& \multicolumn{2}{c}{\text{\otrhat}} \\
\midrule
\textsc{Reacher} & \small -5.2,0.2 & \small -4.2,0.1 & \small -5.8,0.1 & \small -5.1,0.1 & \small -4.8,0.3 & \bf \small -3.7,\bf\small0.1 & \hspace{1mm} (9.0) \\ 
\textsc{Walker2d}  & \small 771.4,76.3 & \small 4600.2,318.1  & \small 3671.7,1196.2  & \small 5178.0,16.3 & \small 4796.7,445.5 & \bf \small 5237.8, \bf 385.0 & \hspace{1mm} \small (8.4) \\
\textsc{Hopper} & \small 1164.8,205.9  & \small 2052.3,185.5 & \small 1646.3,588.2 & \small 3535.3,23.3 & \small 3715.3,61.6 & \bf\small 3749.4,\bf\small60.9 & \hspace{1mm} \small (9.4) \\ 
\textsc{HC}    & \small 1762.7, \small 1056.0 & \small 3485.2,1272.9 & \small 3569.3,50.2    & \small 10772.7,60.2  & \small 7740.5,411.3 & \hspace{1mm} \bf \small11428.2, \bf  \small 70.4 & \hspace{1mm} \small (11.0) \\ 
\textsc{Ant}            & \small 3707.9, \small 135.8  & \small 4505.0,24.3  & \small 4874.7,188.9   & \bf \small 5679.7, \bf 844.3  & \small 4263.0,639.0 & \small 5665.0,240.7 & \hspace{1mm} \small (11.1) \\
\textsc{Swimmer}  & \small 362.6,0.9  & \small 362.0,0.6 & \small 334.7,0.9 & \small 340.3,31.4  & \bf \small 366.3,\bf 1.2  & \small 270.7, 74.9 & \hspace{1mm} (7.6) \\
\textsc{BW}  & \small 192.5, \small 57.2 & \bf \small 300.9, \bf 7.2 & \small 273.0,12.3 & \small 274.3,18.3 &  \small 297.4,13.9  & \small 300.6,14.1 & \hspace{1mm} (12.3) \\
\textsc{Pendulum}   & \small -186.4, \small 14.5   & \small -224.5,30.7   & \small -146.7,4.5  & \small -144.7,3.3  & \bf \small -137.1, \bf 0.7  &  \small-141.9,0.9 & \hspace{1mm} (8.0) \\
\bottomrule
\end{tabular}
\end{center}
\label{tab:RL_Exp}
\end{table*}

\begin{figure*}[h!]
\begin{center}
\begin{minipage}[h]{13cm}
    \begin{algorithmic}[0]
        \IF{$7.78 - 15.7x - 21.7y - 8.48 \dot{\omega} \leq 0$} 
        \RETURN $[3.15, 3.24, 0.65] \cdot P + [0.12, 0.52, 0.04] \cdot I + [10.94, 11.89, 0.13] \cdot D$
        \ELSE
        \RETURN $[3.15, 3.24, 0.65] \cdot P + \theta_I \cdot I + [10.94, 11.89, 0.13] \cdot D$
        \ENDIF
    \end{algorithmic}
\vspace{2mm}

where $\theta_I = [-0.25 + 0.75x + 1.04y + 0.41 \dot{\omega}, 0.52, 0.10 - 0.13x - 0.17y - 0.07 \dot{\omega} ]$
\end{minipage}
\end{center}
\caption{\otr\ PID controller policy for Pendulum. Achieves an average reward of -162.6 over 1000 consecutive episodes.}
\label{fig:policy_pendulum}
\end{figure*}

\paragraph{Table of Results.}

The average reward each algorithm achieves is reported in Table \ref{tab:RL_Exp}, where the best average result for each domain is highlighted in bold. The table also presents the standard deviation of the return on the runs. \otr\ is capable of solving a variety of problems using small actor networks and outperforms oracle-guided approaches (\viper, \lmviper, and \ndps) by a large margin in Walker2D, Hopper, and HC and is overall competitive with \prl. \otr\ has better average reward in six of the eight tested domains, but \prl\ is a close competitor in the six domains. \prl\ outperforms \otr\ in HC and Ant. In contrast to \prl, \otr\ is capable of scaling and training longer programmatic policies. 
\otrhat\ policies trained with 32 neurons perform well in all domains with the exception of Swimmer. The depths of the pruned trees range from 8 to 12, which is a major reduction from the original depth of 32. These results show that \otr\ can synthesize longer and more powerful policies, at the cost of possibly having less interpretable policies.  

\paragraph{Examples of Programs.}
Figure~\ref{fig:policy_mcc} shows an example of a depth-one programmatic policy \otr\ derived for the mountain car domain. In mountain car, the agent controls a car that starts at the bottom of a valley and needs to reach the top of the mountain on the right. The state is defined by the position of the car, $x$, and its velocity, $v_x$. The action is a one-dimensional real number in $[-1, 1]$ and represents the car's acceleration. Simply applying rightward acceleration is not enough to reach the peak; the car must first gain potential energy by moving to the left. The policy is trained using DDPG ~\cite{ddpg} with one hidden neuron and is interpreted as follows. The velocity term dominates the other terms in the decision node, so the equation can be simplified to $v_x \leq 0$. The left leaf node accelerates the car to the left. In the right leaf node, the velocity term again dominates the equation by applying a large acceleration which is clipped to 1. As such, we can simplify the policy and interpret it as follows: if the car is moving to the left, accelerate to the left. Otherwise, accelerate to the right. 

\begin{figure} 
\begin{center}
\advance\leftskip-4mm
\scalebox{0.9}{
\begin{tabular}{c|c} 
\setlength{\tabcolsep}{0.0mm}
\vspace{-1mm}
\begin{minipage}[h]{6.1cm}
    \begin{algorithmic}[0]
        \IF{$-2.2 -3.8x + 114.3v_x \leq 0$}
        \RETURN $-6.1$
        \ELSE
        \RETURN $-102.3 -169.8x + 5116.5 v_x$
        \ENDIF
    \end{algorithmic}
\end{minipage}
& \advance\leftskip-0.4mm
 \begin{minipage}[h]{2.6cm}
    \advance\leftskip-0.4mm
    \begin{algorithmic}[0]
    \advance\leftskip-0.4mm
    \IF{$v_x \leq 0$}
    \RETURN $-1.0$
   \ELSE
   \RETURN $1.0$
   \ENDIF
   \end{algorithmic}
\end{minipage}
\end{tabular}
}
\end{center}
\caption{Original (left) and simplified (right) \otr\ programmatic policies for Mountain Car Continuous. These policies achieve an average reward of $92.8$ and $92.2$, respectively (90 is considered solved).}
\label{fig:policy_mcc}
\end{figure}



Figure \ref{fig:policy_pendulum} shows an example of a depth-one PID controller policy \otr\ derived for pendulum. The state is given by the vector $[x = cos(\omega), y = sin(\omega), \dot\omega]$, where $\omega$ is the angle relative to the upright position and $\dot\omega$ is the angular velocity. The objective is to balance the pendulum upright by applying a leftward or rightward torque, so the stable point is $\epsilon = [1, 0, 0]$. The policy is trained using DDPG with a single hidden neuron and nine output neurons (three for each $\theta_P, \theta_I$ and $\theta_D$) for the actor network. We use aggressive L1 regularization ($\alpha = 2.5$) on the network's output weights so we can learn a sparse model. The resulting policy achieves an average reward of -162.6 over 1000 consecutive episodes. Note that L1 regularization is effective in eliminating the dependence of the policy on the parameters $\theta_P$ and $\theta_D$.

\section{Discussion and Conclusions}

In this paper, we showed that a connection between ReLU neural networks and oblique decision trees can be used to translate neural policies learned with actor-critic methods into programmatic policies. This is because oblique decision trees can be seen as programs written in a commonly used domain-specific language for encoding programmatic policies. We showed that this mapping can handle discrete and continuous output tasks
and can train sparse oblique trees by training ReLU networks with a single hidden layer and L1 regularization. 
The mapping from ReLU neural networks to programs offers a drastically simpler method of synthesizing programmatic policies. 
All one needs to do is encode the actor in a ReLU network so that \otr\ is able to convert it to a programmatic policy. \prl\ is also able to use actor critic algorithms, but the time complexity for evaluating the model during training is exponential on the depth of the tree. \otr's cost is only linear because the model is represented as a neural network during training, which is exponentially smaller than the underlying tree of the model. Despite being simpler than \prl, we showed that the programmatic policies \otr\ synthesizes are competitive with those \prl\ synthesizes if the programs are small; \otr\ outperforms \prl\ in almost all control problems evaluated if it is allowed to synthesize longer programs.

\paragraph{Limitations. }
The programs that we can induce with ReLU networks only support affine transformations of the input data and if-then-else structures. By contrast, program synthesis methods can handle more complex languages, including those with loops (e.g.,~\cite{sketch-sa}). 
Some problem domains might benefit from a language with loops to allow repetitive actions (e.g., ``while not facing a wall, move to the right'') such as the car domain of~\citeauthor{Inala2020Synthesizing}~\shortcite{Inala2020Synthesizing}. Another limitation of \otr\ is that the program structure (i.e., the network architecture) must be provided by the user. By contrast, \prl\ finds not only a policy, but also the structure of the program that encodes the policy~\cite{qiu2022programmatic}. 
Although the main objective of our research is to generate interpretable classifiers and policies, we did not evaluate the interpretability of our models. Similarly to previous work (e.g.,~\cite{Lee2020Oblique,qiu2022programmatic}), we assumed that the models represented with oblique trees are inherently interpretable. Future research needs to evaluate the interpretability of these programs. 


\bibliographystyle{named}
\bibliography{ijcai23}

\end{document}